\documentclass[sigconf]{acmart}
\AtBeginDocument{%
  }

\setcopyright{acmlicensed}
\copyrightyear{2018}
\acmYear{2018}
\acmDOI{XXXXXXX.XXXXXXX}
\acmConference[Conference acronym 'XX]{Make sure to enter the correct
  conference title from your rights confirmation email}{June 03--05,
  2018}{Woodstock, NY}
\acmISBN{978-1-4503-XXXX-X/2018/06}

\usepackage[utf8]{inputenc} 
\usepackage[T1]{fontenc}    
\usepackage{hyperref}       
\usepackage{url}            
\usepackage{booktabs}       
\usepackage{amsfonts}       
\usepackage{nicefrac}       
\usepackage{microtype}      
\usepackage{xcolor}         
\usepackage{enumitem}
\setlist[itemize]{noitemsep}
\usepackage{adjustbox}
\usepackage{algorithm}
\usepackage{algpseudocode}
\usepackage{amsmath}
\usepackage{amsthm}

\newtheorem{lemma}{Lemma}
\usepackage{graphicx}
\usepackage{multirow}
\usepackage{wrapfig} 
\usepackage{subcaption}

\algdef{SE}[VARIABLES]{Variables}{EndVariables}
   {\algorithmicvariables}
   {\algorithmicend\ \algorithmicvariables}
\algnewcommand{\algorithmicvariables}{\textbf{Declare}}

\usepackage{listings}
\theoremstyle{definition}
\newtheorem{definition}{Definition}[section]

\theoremstyle{plain}

\usepackage[table]{xcolor}
\newcommand{\NA}{\cellcolor{black!8}\textsc{N/A}}

\begin{document}

\title{Beyond Single Pass, Looping Through Time: KG-IRAG with Iterative Knowledge Retrieval}

\author{Ruiyi Yang}
\affiliation{%
  \institution{University of New South Wales}
  \city{Sydney}
  \state{NSW}
  \country{Australia}
}
\email{ruiyi.yang@student.unsw.edu.au}

\author{Hao Xue}
\affiliation{%
  \institution{University of New South Wales}
  \city{Sydney}
  \state{NSW}
  \country{Australia}
}
\email{hao.xue1@unsw.edu.au}

\author{Imran Razzak}
\affiliation{%
  \institution{Mohamed Bin Zayed University of Artificial Intelligence}
  \city{Abu Dhabi}
  \country{UAE}
}
\email{imran.razzak@mbzuai.ac.ae}

\author{Hakim Hacid}
\affiliation{%
  \institution{Technology Innovation Institute}
  \city{Abu Dhabi}
  \country{UAE}
}
\email{hakim.hacid@tii.ae}

\author{Flora D. Salim}
\affiliation{%
  \institution{University of New South Wales}
  \city{Sydney}
  \state{NSW}
  \country{Australia}
}
\email{flora.salim@unsw.edu.au}

\renewcommand{\shortauthors}{Trovato et al.}

\begin{abstract}

Retrieval-augmented generation (RAG) has improved large language models (LLMs) on knowledge-intensive tasks, yet most systems assume static facts and struggle when answers depend on serialized and dynamic data, like time--e.g., ordering events, aligning facts to valid intervals, or planning actions under evolving conditions. This paper presents \textbf{K}nowledge-\textbf{G}raph \textbf{I}terative \textbf{R}etrieval-\textbf{A}ugmented \textbf{G}eneration (\textbf{KG-iRAG}), a framework specialized for \textbf{temporal reasoning}. KG-iRAG couples a time-aware planner with a knowledge graph (KG) to iteratively fetch and compose evidence along a temporal axis. Concretely, it (i) represents events and facts with explicit timestamps and validity intervals; (ii) propagates temporal constraints through iterative retrieval using operators; and (iii) verifies temporal consistency while refining intermediate hypotheses, enabling step-by-step deduction for queries that mix knowledge retrieval with inference. Across public temporal QA benchmarks, KG-iRAG consistently improves accuracy and calibration over strong RAG baselines while reducing unnecessary retrieval through targeted, constraint-guided steps. To stress-test real-time decision queries, three application-oriented datasets (\textit{weatherQA-Irish},~\textit{weatherQA-Sydney}, and~\textit{trafficQA-TFNSW}) are additionally constructed and tested alongside existing temporal benchmarks. The results demonstrate that injecting temporal structure into KG-driven RAG yields robust gains on multi-step, time-dependent queries, advancing the state of temporal reasoning with LLMs.
\end{abstract}



\begin{teaserfigure}
\centering         
\includegraphics[width=0.95\textwidth]{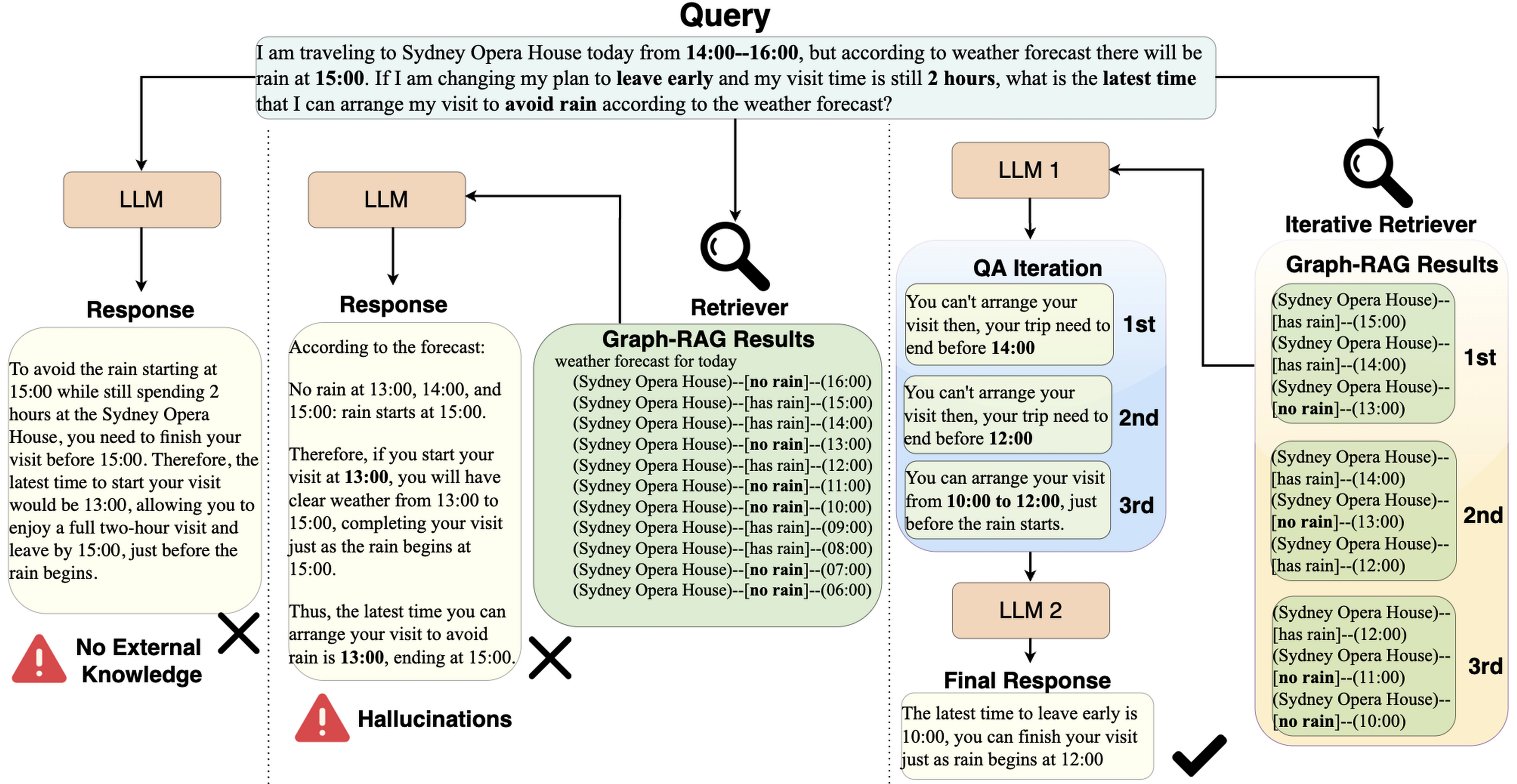}
\caption{A comparison of LLM performance: Direct LLM output, RAG, and Iterative RAG. Without external knowledge, LLMs generate responses solely based on the provided context. However, when presented with excessive data retrieved by RAG (especially numerical information), LLMs are prone to generating hallucinated content.}
\label{KG-iRAG example}
\Description{Banner for KG-iRAG.}
\end{teaserfigure}



\maketitle

\section{Introduction}
While State-of-the-art LLMs still suffered from limitations like fabricating fact or giving outdated answers for questions without answer or require latest data~\cite{kandpal2023large, zhang2023siren}, recent advances in Retrieval-Augmented Generation (RAG) offered a solution by integrating external knowledge sources like KGs with large language models (LLMs), the reliable and flexible supplementary knowledge helps LLMs generate more accurate replies to those queries that require domain-specific information~\citet{zhang2024raft, siriwardhana2023improving}, without extra time for fine-tuning. As a variant of RAG, GraphRAG utilizes KGs to enhance the retrieval process through relational data.  KGs capture intricate relationships between entities, allowing for more contextually rich predictions, thus offering more precise and contextually relevant responses from LLMs~\citet{wu2024medical, matsumoto2024kragen, edge2024local}. Although RAG achieves effectiveness in improving retrieval, they often fail to address the challenges posed in time-related tasks, such as planning trips or making decisions based on weather or traffic conditions. Dealing those queries, current RAG framework still faces several challenges:

$\mathbf{C_1}:\textbf{Accurate Retrieval}$ LLMs need to generate a clear data retrieving plan, avoiding both insufficient(wrong answer) and redundant(waste of token) data recall~\citet{cuconasu2024power}. 

$\mathbf{C_2}:\textbf{Hallucination}$ Even with correct extra data input, facing too many data in prompts, especially numbers, LLMs tend to generate hallucinations and not following correct facts~\cite{gao2023enabling}. 

$\mathbf{C_2}:\textbf{Context-aware Reasoning}$ LLMs need complex reasoning to analyze the data input, finding the accurate patterns behind numbers to get the final correct answers solving the temporal queries~\cite{huang2022towards}. 

Figure~\ref{banner} shows a scenario in which LLMs tend to generate hallucinations facing trip planning. To do such complex reasoning on huge amount of data input, existing RAG methods primarily focus on one-time knowledge retrieval rather than iterative retrieval. To address these limitations, this paper introduces a novel framework, \textbf{Knowledge Graph-Based Iterative Retrieval-Augmented Generation (KG-IRAG)}. KG-IRAG integrates KGs with an iterative reasoning process, enabling LLMs to solve problems by incrementally retrieving relevant data through multi-step queries and achieve step-by-step reasoning. The proposed KG-IRAG framework is validated through experiments on datasets containing queries involving temporal dependencies, demonstrating its ability to significantly enhance LLM performance. Specifically, our key contributions are:

\begin{itemize}
    \item \textbf{Iterative RAG Framework}: The new framework enables LLM reasoning in time-sensitive KG subgraphs through cyclic query refinement, offering a robust and efficient solution for addressing complex spatial temporal queries. The context-aware information retrieval via graph walks mirrors human-like reasoning patterns. 
    
    \item \textbf{New datasets targeting complex temporal queries}: \textit{weatherQA-Irish}, \textit{weatherQA-Sydney} and \textit{trafficQA-TFNSW} are designed to test LLM's ability to answer queries that require both \textbf{retrieving uncertain length of temporal information and doing mathematical reasoning}.
    
\end{itemize}

\section{Related Work}

\subsection{Combination of Graphs with LLMs}

The combinations of LLMs and graphs achieves mutual enhancement. The combination of LLMs with Graph Neural Networks (GNNs) has been shown to significantly improve the modeling capabilities of graph-structured data. LLMs have been shown to contribute to knowledge graph completion, aiding in downstream tasks such as node classification~\citet{chen2023label} and link prediction~\citet{shu2024knowledge}. Additionally, LLMs play a pivotal role in knowledge graph creation, transforming source texts into graphs~\citet{edge2024local, zhang2024causal}. LLMs also enhance performance in Knowledge Graph-based Question Answering (KBQA) tasks, which leverage external knowledge bases to answer user queries~\citet{cui2019kbqa, wu2019survey, fu2020survey}. KBQA tasks includes text understanding and fact-checking~\citet{chen2019tabfact, suresh2024overview}, and approaches are generally categorized into two types: Information Retrieval (IR)-based and Semantic Parsing (SP)-based. IR-based methods directly retrieve information from KG databases and use the returned knowledge to generate answers~\citet{jiang2022unikgqa,jiang2024kg}, whereas SP-based methods generate logical forms for queries, which are then used for knowledge retrieval~\citet{chakraborty2024multi, fang2024dara}. Advanced techniques such as Chain of Knowledge (CoK)~\citet{li2023chain}, G-Retriever~\citet{he2024g}, and Chain of Explorations~\citet{sanmartin2024kg} have been developed to enhance the precision and efficiency of data retrieval from KGs.

\subsection{Retrieval-Augmented Generation (RAG)}

RAG enhances the capabilities of LLMs by integrating external knowledge sources during the response generation process. Unlike traditional LLMs, which rely solely on their pre-trained knowledge, RAG enables models to access real-time or domain-specific information from external databases and knowledge sources. Recent studies have explored RAG from various perspectives, including the modalities of databases~\citet{zhao2024retrieval}, model architectures, training strategies~\citet{fan2024survey}, and the diverse applications of RAG~\citet{gao2023retrieval}. Effective evaluation of RAG systems requires attention to both the accuracy of knowledge retrieval and the quality of the generated responses~\citet{yu2024evaluation}.

Compared to traditional RAG systems, Graph Retrieval-Augmented Generation (GraphRAG) offers a distinct advantage by retrieving knowledge from graph databases, utilizing triplets as the primary data source~\citet{peng2024graph}. Graph-structured data capture relationships between entities and offer structural information, enabling LLMs to interpret external knowledge more effectively~\citet{hu2024grag, bustamante2024sparql, sanmartin2024kg}.

\subsection{LLMs Temporal Reasoning}

Recent advancements have increasingly focused on improving the temporal reasoning capabilities of LLMs. Temporal reasoning in natural language processing (NLP) typically falls into three categories: temporal expression detection and normalization~\citet{li2024sensorllm}, temporal relation extraction, and event forecasting~\citet{yuan2024back}. The integration of temporal graphs has enabled LLMs to perform more effectively in tasks such as time comparison~\citet{xiong2024large}. Several temporal QA datasets have been developed to test LLMs' temporal reasoning abilities, including TEMPLAMA~\citet{dhingra2022time,tan2023towards}, TemporalWiki~\citet{jang2022temporalwiki}, and time-sensitive QA datasets~\citet{chen2021dataset}. By combining LLMs with temporal KGs~\citet{lee2023temporal, yuan2024back, xia2024enhancing},  more accurate event forecasting has become possible.

While these methodologies have shown promising results, certain limitations remain: (1) few GraphRAG methods address queries highly dependent on temporal reasoning, and (2) no existing temporal QA dataset requires consecutive retrieval of uncertain amounts of data from a temporal knowledge base. The proposed KG-IRAG systemaims to bridge these gaps by introducing a novel \textbf{iterative GraphRAG method} and developing datasets specifically focused on \textbf{dynamic temporal information retrieval}.

\begin{figure} 
\centering         
\includegraphics[width=0.45\textwidth]{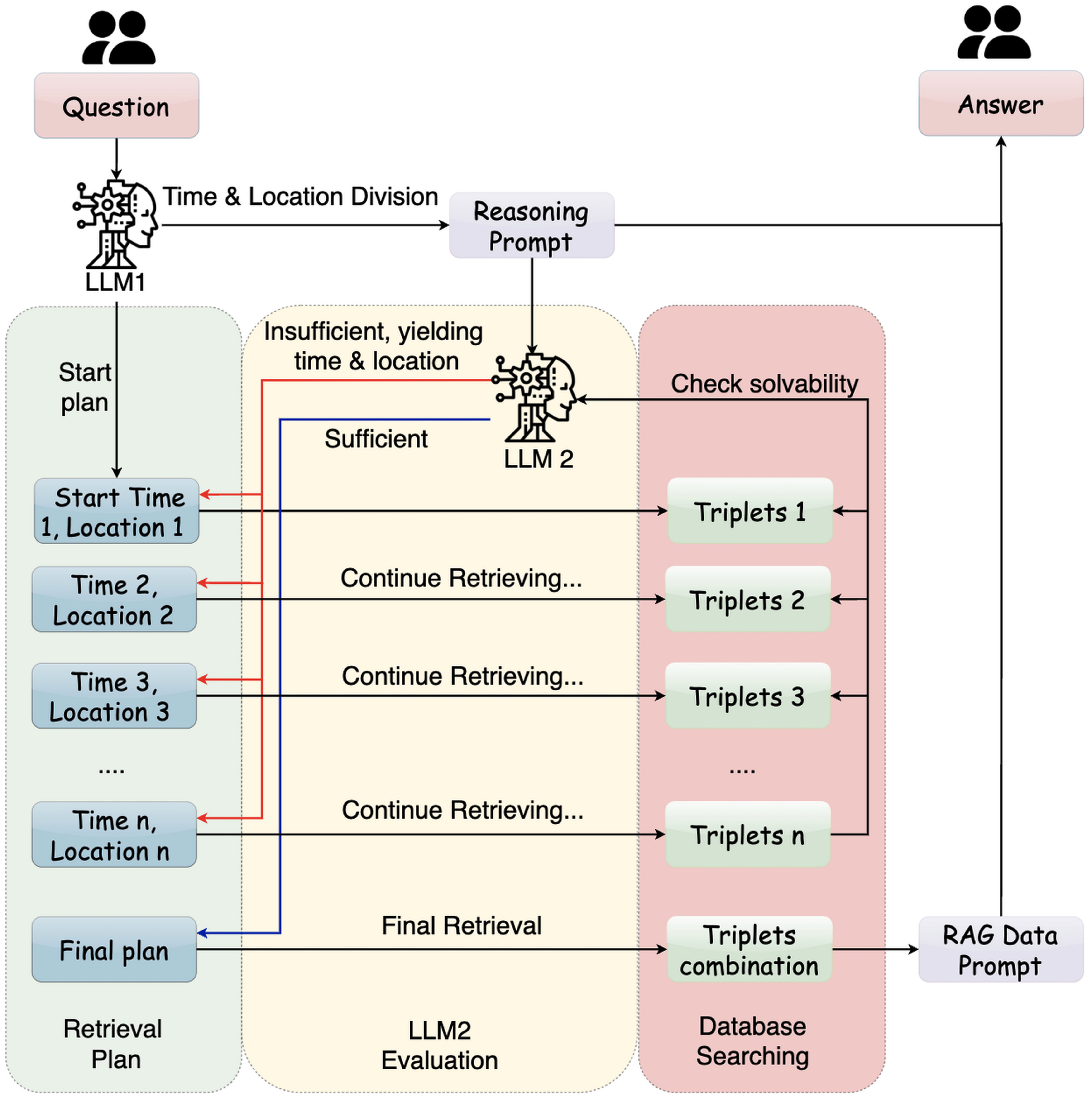}
\caption{The KG-IRAG framework: LLM1 generates an initial retrieval plan and a reasoning prompt, guiding LLM2 through iterative retrievals or stopping to generate the final answer.}
\label{KG-IRAG framework}
\end{figure}

\section{Beyond Single Pass, Looping Through Time: KG-IRAG Framework}

This section presents the KG-iRAG framework for temporal reasoning. The key idea is to \emph{anchor} retrieval on an explicit time point (or interval) and location, iteratively expand evidence along the temporal axis, and stop when a temporal-consistent answer can be verified. Two lightweight LLM agents collaborate throughout: \textbf{LLM1} proposes or updates the temporal/spatial anchors and the retrieval pattern; \textbf{LLM2} judges sufficiency and temporal consistency, and either requests another step or finalizes the answer.

\subsection{Setup and Iterative Overview}

Let $q$ be a natural-language query whose answer depends on time. A reasoning state at iteration $t$ is

\begin{equation}
S^{(t)} = \big(\tau^{(t)},\, \ell^{(t)},\, R^{(t)},\, p^{(t)}\big)   
\end{equation}

where $\tau^{(t)}$ is the current time anchor (a timestamp or interval), $\ell^{(t)}$ is the location context, $R^{(t)}$ is the set of retrieved KG triples at step $t$, and $p^{(t)}$ is a concise retrieval pattern (slot-filled from the query, e.g., entity types, relation hints, required temporal relation such as \textsc{before}/\textsc{during}/\textsc{after}).

Given a temporal KG $D$ (Sec.~\ref{sec:temporal-kg}), a time-aware retrieval operator $\Psi$ materializes evidence around the anchors:
\begin{equation}
\label{eq:retrieval}
\begin{aligned}
R^{(t)} 
&= \Psi\!\left(D;\, \tau^{(t)}, \ell^{(t)}, p^{(t)}\right) \\
&= \big\{\, h \in D \;\big| \;
\underbrace{\mathrm{time}(h)\in \mathcal{W}(\tau^{(t)})}_{\text{temporal window}} \wedge\;
\underbrace{\mathrm{dist}(\mathrm{loc}(h),\ell^{(t)})\le r}_{\text{spatial radius}} \\[-1mm]
&\hspace{3.2em}\wedge\;
\underbrace{\mathrm{match}(h,p^{(t)})}_{\text{pattern}}
\,\big\}.
\end{aligned}
\end{equation}

where $\mathcal{W}(\tau^{(t)})$ is a fixed or adaptive window around the anchor, $r$ is a spatial radius, and $\mathrm{match}$ enforces simple structural filters (e.g., relation type, entity role) and the required temporal relation implied by $p^{(t)}$.

LLM2 evaluates whether the accumulated evidence $R^{(\le t)}=\bigcup_{i=0}^{t}R^{(i)}$ suffices to answer $q$ and is temporally consistent. A simple, interpretable stop rule combining confidence and consistency is used:
\begin{equation}
\small
\label{eq:stop}
\mathrm{stop}^{(t)} \;=\; \Big[\underbrace{\mathrm{conf}\!\left(\mathrm{LLM2}(q, R^{(\le t)})\right)\ge \theta}_{\text{sufficiency}}\Big]\;\wedge\;
\Big[\underbrace{\mathrm{AllenConsistent}\!\left(R^{(\le t)}\right)}_{\text{interval constraints}}\Big],
\end{equation}
where $\mathrm{AllenConsistent}(\cdot)$ is a lightweight check that the time intervals implied by $R^{(\le t)}$ satisfy the expected \textsc{before/after/during/overlap} relations for $q$.

If $\mathrm{stop}^{(t)}=\text{false}$, LLM1 proposes the next anchors by scoring nearby time/location candidates using a utility that prefers \emph{new} temporal coverage with low redundancy:
\begin{equation}
\label{eq:update}
\begin{aligned}
(\tau^{(t+1)}, \ell^{(t+1)}, p^{(t+1)}) 
&= \arg\max_{(\tau', \ell') \in \mathcal{N}_\tau \times \mathcal{N}_\ell} 
Q(\tau', \ell'), \\[1mm]
Q(\tau', \ell') 
&= \underbrace{\mathrm{gain}\!\left(q, R^{(\le t)}, \tau', \ell'\right)
- \lambda\,\mathrm{overlap}\!\left(R^{(\le t)}, \tau', \ell'\right)}_{\text{LLM1-estimated utility}}.
\end{aligned}
\end{equation}

where $\mathcal{N}_\tau$/$\mathcal{N}_\ell$ are local temporal and spatial neighborhoods (e.g., adjacent time windows / nearby stations). This keeps the math minimal while making the policy explicit.

\begin{algorithm}[H]
\caption{KG-iRAG Time-Anchored Iterative Retrieval}
\label{alg:kgirag}
\begin{algorithmic}[1]
\Require Query $q$, temporal KG $D$, thresholds $(\theta, T_{\max})$
\Ensure Final answer $\textit{ans}$
\State $(\tau^{(0)},\ell^{(0)},p^{(0)}) \leftarrow \mathrm{LLM1\_Init}(q)$ \Comment{initial anchors \& pattern}
\For{$t=0$ to $T_{\max}$}
  \State $R^{(t)} \leftarrow \Psi\!\left(D;\tau^{(t)},\ell^{(t)},p^{(t)}\right)$ \Comment{Eq.~\ref{eq:retrieval}}
  \State $(\textit{ok},\, \widehat{\textit{ans}}) \leftarrow \mathrm{LLM2\_Judge}\!\left(q, R^{(\le t)}\right)$
  \If{$\textit{ok}$ \textbf{and} $\mathrm{AllenConsistent}\!\left(R^{(\le t)}\right)$} \Comment{Eq.~\ref{eq:stop}}
     \State \textbf{return} $\widehat{\textit{ans}}$
  \EndIf
  \State $(\tau^{(t+1)},\ell^{(t+1)},p^{(t+1)}) \leftarrow \mathrm{LLM1\_Update}\!\left(q, R^{(\le t)}\right)$ \Comment{Eq.~\ref{eq:update}}
\EndFor
\State \textbf{return} $\mathrm{LLM2\_Fallback}\!\left(q, R^{(\le T_{\max})}\right)$
\end{algorithmic}
\end{algorithm}

\noindent\textbf{Practical notes.} (i) $p^{(t)}$ is a short, human-readable template (entities/relations and required temporal predicate), kept stable unless LLM2 flags mismatch; (ii) $\mathcal{W}(\cdot)$ can shrink when confidence rises to avoid over-retrieval; (iii) $T_{\max}$ bounds latency; in practice $3$--$6$ steps suffice for our tasks.

\subsection{Temporal KG Construction}
\label{sec:temporal-kg}

Data is presented as a directed, typed multigraph $D=(V,E)$ with explicit time:
\begin{itemize}
\item \textbf{Nodes} $V = V_{\text{time}} \cup V_{\text{loc}} \cup V_{\text{event}} \cup V_{\text{value}}$,
where $V_{\text{time}}$ holds time \emph{points} or \emph{intervals} $[t_s,t_e]$, $V_{\text{loc}}$ are geo-entities (e.g., station/road), $V_{\text{event}}$ are occurrences (e.g., \texttt{rain}, \texttt{traffic}), and $V_{\text{value}}$ stores measured attributes (e.g. rain, vehicle count).
\item \textbf{Edges} $E \subseteq V\times \mathcal{R}\times V$ with small, interpretable relation set $\mathcal{R}$ containing \texttt{occursAt}, \texttt{atLocation},\texttt{hasValue}, \texttt{before}, \texttt{after}, \texttt{during}, \texttt{overlaps}, \texttt{near}.

Temporal relations follow Allen’s interval algebra; \texttt{near} is a symmetric spatial proximity.
\end{itemize}

Raw records $d_i$ are normalized by a deterministic mapper $\Gamma$:
\[
\Gamma(d_i)=\big\{(e_i,\texttt{occursAt},\tau_i),\ (e_i,\texttt{atLocation},\ell_i),\ (e_i,\texttt{hasValue},v_i)\big\}
\]
Storing validity intervals on $V_{\text{time}}$ keeps retrieval simple (window checks reduce to interval tests), while derived \texttt{before}/\texttt{during}/\texttt{overlaps} edges allow fast constraint filtering inside $\Psi$ (Eq.~\ref{eq:retrieval}). This design makes time first-class, enables local temporal expansion around anchors, and supports lightweight consistency checks (Eq.~\ref{eq:stop}) without heavy probabilistic machinery. Since application-oriented datasets (weather/traffic) are ingested by the same $\Gamma$; no dataset-specific logic is required in the reasoning loop.

\subsection{Triplet Retrieval} \label{method_TR}

Given the current state $S^{(t)}=(\tau^{(t)},\ell^{(t)},R^{(t)},p^{(t)})$, triplet retrieval instantiates the time-aware operator in Eq.~\ref{eq:retrieval}.
Intuitively, evidence is collected from a temporal window around $\tau^{(t)}$ and a spatial radius around $\ell^{(t)}$, while enforcing the structural/temporal pattern $p^{(t)}$.

\paragraph{Lightweight implementation.}
The retrieval window $\mathcal{W}(\tau^{(t)})$ and radius $r$ start small and grow only when needed.
To bound latency, the neighborhood depth is capped by $k_0\in\{1,2\}$ (local hops), with exact-dedup over previously seen triples to avoid re-fetch.
A simple priority order is used when multiple candidates satisfy the pattern: (i) exact time match $\Rightarrow$ (ii) within-window interval match $\Rightarrow$ (iii) nearest neighbor in time; ties are broken by spatial constrains on $\ell^{(t)}$.
This yields compact batches $R^{(t)}$ that favor temporally on-point evidence. For temporal KGs, each $h\in R^{(t)}$ typically contains
\texttt{(event, occursAt, time)} / \texttt{(event, atLocation, loc)} / \texttt{(event, hasValue, v)}
and optionally \texttt{before} \texttt{/after} \texttt{/during} \texttt{/overlaps} edges derived from validity intervals.
These derived edges permit fast pruning inside $\Psi$ without additional model calls.

\subsection{Iterative Reasoning} \label{method_IR}

After accumulating $R^{(\le t)}=\bigcup_{i=0}^{t}R^{(i)}$, LLM2 proposes a candidate answer and returns a sufficiency score together with a temporal-consistency verdict.
Stopping follows Eq.~\ref{eq:stop}: confidence must exceed $\theta$ \emph{and} the Allen-style interval constraints implied by $R^{(\le t)}$ must be satisfied.

If not sufficient, the controller updates anchors using Eq.~\ref{eq:update}.
In practice, the utility prefers (a) new time coverage adjacent to $\tau^{(t)}$, (b) nearby locations along the same line/road, and (c) patterns that fill the missing predicate needed by the query (e.g., switch from \textsc{during} to \textsc{before} if a latest-feasible-time is sought).
This induces a measured outward expansion in time and space, rather than a wide one-shot sweep.

\begin{algorithm}[H]
\caption{Temporal Loop Controller (LLM1 + LLM2)}
\label{alg:controller}
\begin{algorithmic}[1]
\Require Query $q$, KG $D$, thresholds $\theta$, $T_{\max}$
\State $(\tau^{(0)},\ell^{(0)},p^{(0)}) \leftarrow \mathrm{LLM1\_Init}(q)$
\For{$t=0$ to $T_{\max}$}
  \State $R^{(t)} \leftarrow \Psi\!\left(D;\tau^{(t)},\ell^{(t)},p^{(t)}\right)$
  \State $(ok,\widehat{ans}) \leftarrow \mathrm{LLM2\_Judge}\!\left(q, R^{(\le t)}\right)$
  \If{$ok$ \textbf{and} $\mathrm{AllenConsistent}\!\left(R^{(\le t)}\right)$}
     \State \textbf{return} $\widehat{ans}$
  \EndIf
  \State $(\tau^{(t+1)},\ell^{(t+1)},p^{(t+1)}) \leftarrow \mathrm{LLM1\_Update}\!\left(q, R^{(\le t)}\right)$
\EndFor
\State \textbf{return} \text{``no answer''} \Comment{budget reached}
\end{algorithmic}
\end{algorithm}

\noindent\textit{Interpretation.}
LLM2 acts as a temporal verifier: it accepts only when the proposed answer is both well-supported and interval-consistent.
LLM1 acts as a planner: it shifts anchors to reduce uncertainty with minimal new retrieval.

\subsection{Answer Generation} \label{method_Gen}

When the stop rule fires, the final answer is produced from the consolidated evidence $E=\mathrm{Dedup}\!\left(R^{(\le t)}\right)$ together with the query $q$.
A simple, transparent fusion is used: facts close to the \emph{decisive time} receive higher weight.

\begin{equation}
\label{eq:fuse}
\begin{aligned}
ans 
&= \mathrm{LLM2\_Synthesize}\!\left(q,\; \sum_{h\in E}\alpha(h)\cdot h\right), \\[1mm]
\alpha(h) 
&\propto \exp\!\big(-\gamma\cdot \mathrm{dist\_time}(h,t^\star)\big)
\cdot \mathbf{1}\!\{\mathrm{match}(h,p^{(\le t)})\}.
\end{aligned}
\end{equation}

where $t^\star$ is selected from $E$ as the time point/interval that satisfies the query’s temporal predicate (e.g., latest feasible departure), and $\gamma$ controls temporal decay (default $0.5$--$1.0$).
The synthesis call returns both the answer and a short natural-language rationale citing the highest-weight facts.
If no $t^\star$ can be identified or the retrieval budget $\kappa$ is exceeded, the systemreturns \text{``no answer''}.

\paragraph{Robustness to noisy evidence.}
Before Eq.~\ref{eq:fuse}, a contradiction filter removes pairs of facts that violate the accepted interval constraints for $q$ (e.g., mutually exclusive values at the same time).
This keeps the fused context compact and consistent.

\subsection{Why Not Single-Pass Retrieval?} \label{subsec:why-not-single}

Single-pass RAG must choose a large temporal window and spatial neighborhood \emph{a priori}.
In temporal KGs, the effective branching factor $\delta$ (neighbors per hop) compounds with the \emph{dependency depth} $d$ (how many anchored steps are needed to satisfy the temporal predicate), leading to a cost that scales as
\begin{equation}
\label{eq:cost}
\mathrm{Cost}_{\text{single}} \;\approx\; \Theta\!\big(\delta^{\,d}\big),
\end{equation}
or it risks missing critical intervals if the window is narrowed.
By contrast, KG-iRAG executes $T$ controlled steps with small local radius $k_0$ (typically $k_0\!\le\!2$), yielding
\[
\mathrm{Cost}_{\text{iter}} \;\approx\; T\cdot \Theta\!\big(\delta^{\,k_0}\big), \quad \text{with } k_0 \ll d.
\]
This reduces extraneous retrieval while preserving accuracy: anchors are moved only when evidence indicates where additional temporal coverage is most useful.
Empirically, $T\in[3,6]$ suffices for the studied tasks, avoiding the accuracy–efficiency dilemma inherent to single-pass sweeps.

\subsection{Case Study}

To further explain the details of ‘iteration‘ in KG-IRAG, a case study was conducted by simulating an initial travel plan to Sydney Opera House, including defining whether the initial plan is valid, as well as how to define an optimal time to adjust the trip. Considering the following Q1, Q2 and related data:

Q1: \textit{`I plan to visit the Sydney Opera House from 13:00 to 15:00 on December 5th. Considering the weather, can I avoid rain during the trip? `}

Q2: \textit{`If I can't avoid the rain, what is the earliest time I can postpone my trip to avoid rain? `}

\paragraph{Temporal series (Dec 5, Sydney Opera House).}
13{:}00 cloudy; \textbf{13{:}30 rain}; 14{:}00 cloudy; 14{:}30 cloudy; 15{:}00 cloudy; \textbf{15{:}30 rain}; \textbf{16{:}00 rain}; 16{:}30 cloudy; 17{:}00 cloudy; 17{:}30 cloudy; 18{:}00 cloudy; 18{:}30 cloudy.

\paragraph{Setup.}
Slot length $\Delta=1$h. The retrieval operator $\Psi$ (Eq.~\ref{eq:retrieval}) uses a temporal window aligned to the current anchor and enforces the pattern $p$ (\textsc{no-rain} for a contiguous window). LLM2 applies the stop rule (Eq.~\ref{eq:stop}); otherwise, anchors are updated by Eq.~\ref{eq:update} to expand coverage with minimal redundancy.

\medskip
\noindent\textbf{Q1 — Fixed window feasibility (13{:}00–15{:}00).}
Initialize $\tau^{(0)}=[13{:}00,15{:}00],\,\ell^{(0)}=\text{Opera House},\,p^{(0)}=\textsc{no-rain-in-window}$.
Applying $\Psi$ retrieves evidence $R^{(0)}$ that includes \textbf{13{:}30 rain}.
LLM2 judges insufficiency for \textsc{no-rain} and returns \textbf{Q1 = No} (fails temporal predicate); the process stops for Q1.

\medskip
\noindent\textbf{Q2 — Earliest postponement with a $2$h dry window.}
Now $p$ requires a contiguous $2$h \textsc{no-rain} interval with the earliest feasible start time $\ge$ 13{:}30.
Iteration proceeds as follows (updates choose the closest future window that increases temporal coverage while avoiding overlap):

\begin{table}[H]
\centering
\small
\begin{adjustbox}{max width=0.47\textwidth}
\begin{tabular}{c|c|l|l}
\toprule
Step $t$ & Anchor window $\tau^{(t)}$ & Key evidence from $\Psi$ & Decision (LLM2 / Controller) \\
\midrule
0 & $[14{:}00,16{:}00]$ & \textbf{15{:}30 rain}, \textbf{16{:}00 rain} & Violates \textsc{no-rain}; update by Eq.~\ref{eq:update} \\
1 & $[16{:}00,18{:}00]$ & \textbf{16{:}00 rain} & Still violates; shift start forward \\
2 & $[16{:}30,18{:}30]$ & 16{:}30–18{:}30 all cloudy & Satisfies pattern; \textbf{stop} (Eq.~\ref{eq:stop}) \\
\bottomrule
\end{tabular}
\end{adjustbox}
\end{table}

The earliest feasible postponement that yields a $2$h rain-free interval is therefore \textbf{16{:}30–18{:}30}.
The controller favors $[16{:}30,18{:}30]$ at Step~2 because it is the first window starting after the last observed rain time (16{:}00) that also meets the $2$h constraint without overlap-induced redundancy.

\paragraph{Answer synthesis.}
Upon stopping, evidence $E=\mathrm{Dedup}(R^{(\le t)})$ is fused by the temporal weighting in Eq.~\ref{eq:fuse}, where the decisive time $t^\star$ is the window start (16{:}30) satisfying the query predicate. The generated rationale cites the rain hits at 13{:}30/15{:}30/16{:}00 and the subsequent cloudy stretch 16{:}30–18{:}30.


\section{Dataset Properties}

\subsection{Data Properties}

Three application-oriented corpora are used to stress temporal reasoning under different sampling cadences:

\begin{itemize}\setlength{\itemsep}{2pt}
\item \textbf{weatherQA-Irish} (Jan 2017–Dec 2019): hourly records from Met Éireann, covering 25 stations across 15 counties.
\item \textbf{weatherQA-Sydney} (Jan 2022–Aug 2024): weather observations collected every 30 minutes in Sydney.
\item \textbf{trafficQA-TFNSW} (2015–2016): hourly traffic volumes from permanent counters/classifiers in Sydney.
\end{itemize}

Table~\ref{attribute-combined} lists the minimal attributes retained to construct temporal KGs and the corresponding QA labels. Time is treated as a first-class entity (points or short intervals) to simplify window checks and downstream temporal constraints. Table~\ref{quant-datasets} reports the corpus sizes after normalization.

\begin{table}[ht]
\centering
\caption{Corpus scale after normalization.}
\setlength{\tabcolsep}{6pt}
\begin{tabular}{cccc}
\toprule
\multicolumn{4}{c}{\textbf{Irish Weather}} \\
\midrule
period & entities & relations & records (raw) \\
\midrule
2017--2019 & 227{,}760 & 876{,}000 & 219{,}000 \\
\midrule
\multicolumn{4}{c}{\textbf{Sydney Weather}} \\
\midrule
period & entities & relations & records (raw) \\
\midrule
2022--2024 & 332{,}433 & 559{,}673 & 279{,}837 \\
\midrule
\multicolumn{4}{c}{\textbf{TFNSW Traffic}} \\
\midrule
period & entities & relations & records (raw) \\
\midrule
2015--2016 & 132{,}042 & 683{,}002 & 683{,}002 \\
\bottomrule
\end{tabular}
\label{quant-datasets}
\end{table}

“Records (raw)” counts the ingested rows prior to KG mapping and deduplication.
Weather corpora contain extra attributes (e.g., humidity, sunlight) and occasional duplicates; these are dropped, so relations can exceed raw rows.
The TFNSW corpus is already compact and fully utilized.

\subsection{Question–Answer Generation} \label{data:qa}

The QA design progressively increases temporal difficulty: \textbf{Q1} is a base detection on a fixed window; \textbf{Q2}/\textbf{Q3} require locating the nearest feasible window before/after an anchor time.
This section only defines labels and generation policy; the iterative reasoning mechanics are covered in the Method section.

\paragraph{Sampling anchors.}
For each KG \(D\), anchor times \(\mathcal{T}_{\text{anchor}}\) are sampled uniformly within the temporal coverage.
A trip duration \(\Delta t\) (e.g., 1--2\,h) and a maximum search horizon \(L\) (e.g., 9--12\,h) are fixed per dataset.

\paragraph{Q1: base detection on a fixed window.}
For anchor \(t\), define the \emph{event-existence indicator}
\[
I_{\mathrm{exist}}(t) \;=\; \mathbb{I}\!\left(\exists\, e \in D\ \text{within}\ [t,\, t+\Delta t]\right).
\]
Two phrasings are supported in the datasets:
(i) \emph{Detection} label \(a_{\mathrm{det}}(t)=I_{\mathrm{exist}}(t)\);
(ii) \emph{Avoidability} label \(a_{\mathrm{avoid}}(t)=\neg I_{\mathrm{exist}}(t)\).
In practice, Q2/Q3 are instantiated only when the planned window is \emph{not} avoidable, i.e., \(I_{\mathrm{exist}}(t)=1\).

\paragraph{Q2/Q3: nearest feasible window before/after the anchor.}
Given the same duration \(\Delta t\) and horizon \(L\),
\[
a_{\text{early}}(t) \;=\; \max \big\{ t' < t \;\big|\; \nexists\, e \in [t',\, t'+\Delta t],\ \ t' > t-L \big\},
\]
\[
a_{\text{late}}(t) \;=\; \min \big\{ t' > t \;\big|\; \nexists\, e \in [t',\, t'+\Delta t],\ \ t' < t+L \big\}.
\]
If no such \(t'\) exists within the horizon, the label is set to \textsc{None}.

\begin{algorithm}[H]
\caption{QA Pair Generation (per KG)}
\label{alg:qa_gen}
\begin{algorithmic}[1]
\Require KG $D$, \#anchors $M$, duration $\Delta t$, horizon $L$
\Ensure QA set $\mathcal{Q}$
\State $\mathcal{Q}\leftarrow\emptyset$, \; $\mathcal{T}_{\text{anchor}} \leftarrow \mathrm{UniformSample}(D.\mathrm{times}, M)$
\For{each $t \in \mathcal{T}_{\text{anchor}}$}
    \State add $(q_1(t), a_{\mathrm{det/avoid}}(t))$
    \If{$I_{\mathrm{exist}}(t)=1$}
        \State compute $a_{\text{early}}(t)$, $a_{\text{late}}(t)$ within $L$
        \State add $(q_2(t), a_{\text{early}}(t))$, $(q_3(t), a_{\text{late}}(t))$
    \EndIf
\EndFor
\end{algorithmic}
\end{algorithm}

\paragraph{Numbers.}
Table~\ref{num-datasets} summarizes the distribution of Q1 labels and the availability of Q2/Q3 answers.
For Q2/Q3, counts are reported only over anchors with \(I_{\mathrm{exist}}(t)=1\) (i.e., the planned window initially violates the constraint).

\begin{table}[ht]
\centering
\caption{QA set statistics. Q1 counts are over all anchors; Q2/Q3 counts are over anchors with \(I_{\mathrm{exist}}(t)=1\).}
\begin{adjustbox}{max width=0.5\textwidth}
\begin{tabular}{cccc}
\toprule
\multicolumn{4}{c}{\textbf{Q1: abnormal event over \([t,t{+}\Delta t]\)}} \\
\midrule
Dataset & \#items & True (event) & False (no event) \\
\midrule
weatherQA-Irish & 600 & 393 & 207 \\
weatherQA-Sydney & 600 & 211 & 389 \\
trafficQA-TFNSW & 400 & 253 & 147 \\
\midrule
\multicolumn{4}{c}{\textbf{Q2\&Q3 (conditioned on Q1=True)}} \\
\midrule
Dataset & seed cases & has answer & no answer \\
\midrule
weatherQA-Irish & 207 & 204 & 3 \\
weatherQA-Sydney & 389 & 339 & 50 \\
trafficQA-TFNSW & 147 & 135 & 12 \\
\bottomrule
\end{tabular}
\end{adjustbox}
\label{num-datasets}
\end{table}

\paragraph{Example (weatherQA–Sydney, 30-min cadence).}

\textit{Anchor.} Location path: [PARRAMATTA NORTH $\rightarrow$ HOLSWORTHY CONTROL RANGE];
start time: \texttt{11-Mar-2024 04:00}; trip duration $\Delta t = 4$h.

\begin{itemize}[leftmargin=1.1em]
\item \textbf{Q1 (avoidability).} 
``I am a resident in PARRAMATTA NORTH, and I plan to have a trip on 11-Mar-2024 04:00, passing ['PARRAMATTA NORTH', 'HOLSWORTHY CONTROL RANGE'] in 4 hours, considering the weather, can I avoid rain during this time?''
\\\textit{Gold:} \textbf{False} (rain at 04{:}00).

\item \textbf{Q2 (latest-before, leave early).}
``Based on the above situation, if I cannot avoid rain, I can leave early, what is the latest time for me to leave early. If I can avoid rain, just answer 'no need', if there is no such time, just answer 'no answer'.''
\\\textit{Gold:} \textbf{no answer} (no rain-free 4h window entirely before 04{:}00).

\item \textbf{Q3 (earliest-after, leave late).}
``Based on the above situation, if I cannot avoid rain, I can leave late until there is no rain, what is the earliest time for me to leave late? If I can avoid rain, just answer 'no need', if there is no such time, just answer 'no answer'.''
\\\textit{Gold:} \textbf{2024-03-11T04{:}30{:}00} (04{:}30–08{:}30 all ``No rain'').
\end{itemize}

\subsection{Dynamic ProblemDecomposition}

Q2/Q3 decompose a global temporal requirement (``find a feasible \(\Delta t\) window closest to the anchor'') into a sequence of fixed-window checks identical in structure to Q1.
This yields a uniform labeling rule across datasets while naturally exercising the time-anchored iteration in method:
anchors advance just beyond violating intervals, and feasibility is certified by simple interval tests.
The design avoids dataset-specific heuristics yet produces queries that require genuine temporal reasoning rather than single-shot lookup.

\section{Experiments}

\subsection{Evaluation}

\paragraph{Tasks and datasets.}
Evaluation covers the three application datasets in Sec.~\ref{data:qa} and one external temporal QA benchmark to show KG-iRAG's usability on other cases. As a KG-based temporal QA set built on Wikidata with explicit/implicit/ordinal temporal categories and standard splits, \textbf{TimeQuestions}~\cite{jia2021complex} (16{,}181 Qs) is adopted for out-of-domain benchmarking.%

\paragraph{Standard Data (minimal evidence).}
For the in-house datasets, each question is paired with a \emph{minimal} evidence set \(SD\) that suffices to answer it:
\begin{equation}\label{best-data}
SD \;=\; \arg\min_{D \subseteq D_{\mathrm{all}}}\ \text{s.t.\ $D$ solves the query},
\end{equation}
\begin{equation}\label{best-data2}
\forall d\in SD,\ \ (SD\setminus\{d\})\ \text{does not solve the query}.
\end{equation}
This ensures fair comparison across formats (raw table, verbalized text, KG triples) without rewarding extraneous context.

\paragraph{Metrics.}
Reporting uses \textbf{Accuracy (Acc)} and \textbf{macro-F1} for answer quality, and a Jaccard-style \textbf{Hit Rate (HR)} for retrieval usefulness (share of retrieved items that intersect $SD$). Formulas for EM/F1 are standard and omitted for brevity. Hallucination is explicitly measured (see below).

\paragraph{Hallucination.}
Hallucination denotes content not supported by the provided evidence (or violating the accepted temporal constraints). The indicator and rate are

\begin{equation} \label{Hallucination}
\small
\text{hal}(answer_i, truth_i) =
\begin{cases} 
1 & \text{hallucination detected} \\
0 & \text{otherwise}
\end{cases}
\end{equation}

\begin{equation} \label{Hal}
\small
  Hallucination(\%) = \frac{\Sigma_{i=0}^N hal(answer_i, truth_i)=1}{N} \\
\end{equation}

\begin{table*}[h]
\centering
\caption{Performance comparison of different methods on KGQA benchmarks with KG(Triplets) input. The \textbf{best} and \underline{second-best} methods are denoted.}
\label{table:baseline-results}
\begin{adjustbox}{max width=\textwidth}
\begin{tabular}{l|cccc|cccc|cccc|cc}
\toprule
\multicolumn{15}{c}{\textbf{Overall Results}} \\ \cline{1-15}
\multirow{2}{*}{Method} & \multicolumn{4}{c|}{WeatherQA-Irish} & \multicolumn{4}{c|}{WeatherQA-Sydney} & \multicolumn{4}{c|}{TrafficQA-TFNSW} & \multicolumn{2}{c}{TimeQuestion} \\
\cmidrule(lr){2-5} \cmidrule(lr){6-9} \cmidrule(lr){10-13} \cmidrule(lr){14-15} 
& Q2-Acc & Q2-F1 & Q3-Acc & Q3-F1 & Q2-Acc & Q2-F1 & Q3-Acc & Q3-F1 & Q2-Acc & Q2-F1 & Q3-Acc & Q3-F1 & Acc & F1 \\
\midrule

\midrule
\rowcolor{black!10}\multicolumn{15}{l}{\textbf{LLM methods (no KG retrieval)}} \\

Llama3-8b
& 0.009 & \NA & 0.048 & \NA
& 0.141 & \NA & 0.149 & \NA
& 0.061 & \NA & 0.087 & \NA
& 0.178 & \NA \\

GPT-3.5-turbo
& 0.034 & \NA & 0.034 & \NA
& 0.129 & \NA & 0.136 & \NA
& 0.054 & \NA & 0.095 & \NA
& 0.258 & \NA \\

GPT-4o-mini
& 0.014 & \NA & 0.058 & \NA
& 0.123 & \NA & 0.149 & \NA
& 0.087 & \NA & 0.102 & \NA
& 0.327 & \NA \\

GPT-4o
& 0.043 & \NA & 0.053 & \NA
& 0.126 & \NA & 0.141 & \NA
& 0.068 & \NA & 0.109& \NA
& 0.354 & \NA \\

Falcon-H1-7B-Instruct
& 0.048 & \NA & 0.072 & \NA
& 0.129 & \NA & 0.147 & \NA
& 0.075 & \NA & 0.116 & \NA
& 0.348 & \NA \\

\midrule
\rowcolor{black!10}\multicolumn{15}{l}{\textbf{KG+LLM methods}} \\

Graph-RAG
& 0.208 & 0.673 & 0.357 & 0.684
& 0.226 & 0.514 & 0.25 & 0.503
& 0.361 & 0.694 & 0.565 & 0.709
& 0.693 & 0.734 \\

KG-RAG~\cite{sanmartin2024kg}
& 0.435 & \underline{0.857} & 0.575 & 0.839
& 0.378 & 0.785 & 0.362 & 0.793
& \underline{0.537} & \underline{0.878} & 0.776 & 0.869
& 0.717 & 0.783 \\

ToG~\cite{sun2023think}
& \underline{0.449} & 0.842 & \underline{0.584} & \underline{0.857}
& \underline{0.411} & \underline{0.833} & \underline{0.398} & \underline{0.824}
& 0.531 & 0.861 & \underline{0.789} & \underline{0.893}
& 0.742 & 0.812 \\

\textbf{KG-iRAG} 
& \textbf{0.459} & \textbf{0.875} & \textbf{0.604} & \textbf{0.88}
& \textbf{0.429} & \textbf{0.852} & \textbf{0.411} & \textbf{0.869}
& \textbf{0.558} & \textbf{0.886} & \textbf{0.803} & \textbf{0.914}
& \textbf{0.781} & \textbf{0.824} \\
\midrule

\end{tabular}
\end{adjustbox}
\end{table*}

Operationally: for \textbf{Q1}, any incorrect yes/no constitutes hallucination; for \textbf{Q2}/\textbf{Q3}, hallucination is flagged if (a) the predicted time window is not in evidence, (b) the predicted abnormal event time is wrong, or (c) the controller fails to stop despite meeting temporal constraints. For \textbf{Q2}/\textbf{Q3} on each dataset, 50 items are randomly sampled for manual verification.

\paragraph{Models and baselines.}
LLMs: Llama-3-8B-Instruct, GPT-3.5-turbo-0125, GPT-4o-mini-2024-07-18, GPT-4o-2024-08-06, Falcon-H1-7b~\cite{zuo2025falcon}.
Baselines:
(i) \textbf{Single-pass Graph-RAG} (one-shot selection of a fixed temporal window, no iteration);
(ii) \textbf{KG-RAG}~\cite{sanmartin2024kg} with three exploration steps following the Chain-of-Explorations retrieval plan; %
(iii) \textbf{Think-on-Graph}~\cite{sun2023think} with knowledge traceability and knowledge correctability by leveraging LLMs reasoning and expert feedback.
(iv) \textbf{KG-iRAG} (proposed).

For \textbf{TimeQuestions}, official train/dev/test splits and Wikidata snapshot are used as provided; KG-iRAG and baselines interface with the KG via the same retrieval API and report EM/F1 only (no $SD$/HR, as gold minimal evidence is not supplied).

\subsection{Experiment Settings}

\paragraph{Stage A — Format sensitivity without iteration.}
To isolate LLM temporal reasoning \emph{given correct evidence}, each question is paired with its $SD$ and rendered in three formats:
\(\{d_{\mathrm{raw}}, d_{\mathrm{text}}, d_{\mathrm{triplet}}\}\).
These are injected into a fixed answer template and fed to each LLM without retrieval/iteration.
This diagnoses whether representation alone affects temporal reasoning on our datasets.

\paragraph{Stage B — End-to-end retrieval and QA.}
Baselines and KG-iRAG are run end-to-end on all datasets.
For KG-RAG (CoE), the exploration depth is capped at 3 steps to match the temporal span of typical questions; %
for ReAct-KG, the agent can interleave short reasoning traces with time-bounded KG actions but has no access to our specialized stop rule; %
for Single-pass, a fixed window is chosen once from the anchor.
KG-iRAG uses the same underlying LLMs for \emph{LLM1} (planner) and \emph{LLM2} (verifier) as in Fig.~\ref{KG-IRAG framework}. EM/F1/HR and hallucination are reported per dataset.

\section{Result}


\subsection{Result analysis}

\subsubsection{Direct inputs (no retrieval/iteration).}
Across models and datasets, structured \textbf{triplets} generally yield the best accuracy in Table~\ref{data_comparison}. Especially for temporally grounded queries (\emph{Q2/Q3}), triplets are consistently stronger: e.g., on \emph{Irish Q2}, GPT-4o reaches 0.396 (triplets) vs.\ 0.3865/0.3314; on \emph{Sydney Q2}, GPT-4o improves to 0.347 (triplets); on \emph{TFNSW Q3}, GPT-4o attains 0.7279 (triplets).
These trends indicate that explicitly exposing \emph{time–location–event} structure reduces spurious reasoning and helps align constraints.

\subsubsection{RAG comparison on Q1.}
In Table~\ref{table:baseline-results}, standard Graph-RAG, KG-RAG, ToG and KG-iRAG perform similarly on \emph{Q1}.
This is expected: when the window is fixed and the predicate is a binary existence test, one-pass retrieval already provides adequate evidence, and iterative control offers limited additional benefit.
Figure~\ref{results} shows the same pattern in the accuracy summaries.

\subsubsection{RAG comparison on Q2/Q3 (temporal search).}
On temporally grounded questions, KG-iRAG consistently ranks ahead of all baselines in Table~\ref{table:baseline-results}. Compared with other baselines, which contians premature stops or lacks explicit temporal windows, KG-iRAG avoids both over- and under-retrieval by updating anchors just beyond violating events and verifying sufficiency with interval consistency before answering. The same behavior is observed on the external \emph{TimeQuestions} benchmark: without dataset-specific heuristics, the time-anchored controller and lightweight consistency checks transfer directly to Wikidata timestamps/validity intervals, preserving the ordering among methods. This suggests the approach is not tailored to the in-house corpora and can be applied to other temporal KGQA datasets with minimal adaptation (i.e., a shared retrieval API and time-normalized triples).

\subsubsection{Hallucination (definition-based auditing).}
Manual audits (50 items per dataset for Q2/Q3) following the definition in Eq.~\ref{Hallucination} show a clear ordering:
\emph{Single-pass Graph-RAG} has the highest hallucination, driven by oversized contexts when the model requests maximal horizons;
\emph{KG-RAG} reduces hallucination but is vulnerable to under-retrieval on longer spans;
\textbf{KG-iRAG} yields low hallucination, attributable to (i) minimal-evidence retrieval around anchors and (ii) explicit interval-consistency checks before answer synthesis.
For Q1, hallucination rates largely mirror EM errors, consistent with the binary definition.

\subsection{Ablation Study}

Two variants are derived from KG-iRAG to diagnose which components drive temporal search quality.
\emph{Limited-Recall} caps the number of iterations to a small fixed budget (2-iteration at most) and disables adaptive window growth, forcing the controller to answer after a short exploratory.
\emph{No-Sufficiency-Check} removes the stop rule in Eq.~\ref{eq:stop}; the controller executes a fixed number of steps and synthesizes an answer without verifying temporal consistency.
Experiments are run on \textit{weatherQA–Irish} (hourly) and \textit{weatherQA–Sydney} (30-min), focusing on Q2/Q3.

\begin{table}[h]
\centering
\caption{Ablation on \textbf{GPT-4o} for temporal search (Q2/Q3). Acc/F1; arrows indicate change vs Full.}
\label{tab:ablation_gpt4o_sideby}
\small

\begin{minipage}{0.62\columnwidth}
\centering
\begin{tabular}{l l c c}
\toprule
\multicolumn{4}{r}{\textbf{weatherQA–Irish}}\\
\midrule
Variant & Task & Acc & F1 \\
\midrule
Full KG-iRAG & Q2 & \textbf{0.459} & \textbf{0.875} \\
Limited-Iteration & Q2 & 0.353 & 0.691 \\
No-Sufficiency-Check & Q2 & 0.391 & 0.733 \\
\midrule
Full KG-iRAG & Q3 & \textbf{0.604} & \textbf{0.880} \\
Limited-Recall & Q3 & 0.488 & 0.626 \\
No-Sufficiency-Check & Q3 & 0.541 & 0.785 \\
\bottomrule
\end{tabular}
\end{minipage}
\hfill
\begin{minipage}{0.37\columnwidth}
\centering
\begin{tabular}{c c c}
\toprule
\multicolumn{3}{c}{\textbf{weatherQA–Sydney}}\\
\midrule
Task & Acc & F1 \\
\midrule
Q2 & \textbf{0.429} & \textbf{0.852} \\
Q2 & 0.326 & 0.614 \\
Q2 & 0.362 & 0.702 \\
\midrule
Q3 & \textbf{0.411} & \textbf{0.869} \\
Q3 & 0.296 & 0.598 \\
Q3 & 0.347 & 0.703 \\
\bottomrule
\end{tabular}
\end{minipage}

\end{table}

Both variants underperform the full system on earliest / earliest-after and latest / earliest-before queries.
Capping iterations induces under-coverage near the decisive boundary, particularly when the first feasible window lies just beyond the last violating event; answers then lean conservative for “latest-before” and overly optimistic for “earliest-after”.
Removing the sufficiency check produces the opposite failure mode: retrieval drifts into redundant windows, the fused context becomes harder to calibrate, and hallucination rises when the synthesis step commits to unsupported intervals.
The effect is more visible in higher-frequency data where windows are narrower and violations recur in short succession, as in the Sydney corpus.


\subsection{Latency Analysis on TimeQuestions}

End-to-end wall-clock latency per question is measured on \textit{TimeQuestions}, comparing KG-iRAG with ToG under the same LLM backbone, KG backend, and token budgets.
Latency includes planning, retrieval calls, and answer generation, with caching disabled.

\begin{table}[h]
\centering
\caption{Latency on \textit{TimeQuestions} with \textbf{GPT-4o} (same KG backend, token budgets, caching disabled). All values are estimated under a unified setup.}
\label{tab:latency_timequestions_gpt4o}
\begin{adjustbox}{max width=\columnwidth}
\begin{tabular}{lccc}
\toprule
Method & Median (s) & Avg \# LLM Calls & Avg \# KG Calls \\
\midrule
KG-iRAG(Ours) & 4.9 & 4.2 & 2.1\\
KG-RAG & 5.6  & 4.6 & 3.3 \\
ToG & 12.8 & 6.5 & 8.4  \\
\bottomrule
\end{tabular}
\end{adjustbox}
\end{table}

KG-iRAG exhibits lower latencies compared with other baselines.
Anchors converge in a small number of steps, and the stop rule halts early once the required interval constraints are satisfied, keeping retrieval batches compact.
ToG’s iterative planning remains competitive on multi-hop relational paths, but when queries hinge on locating the nearest feasible time window, beam expansions and re-planning loops introduce additional tool calls and longer traces, which accumulates latency.

The observed latency pattern aligns with results on the in-house temporal datasets: when answers depend on verifying a short interval around an anchor rather than exploring long relational paths, time-anchored iteration with verifiable stopping maintains accuracy while avoiding prolonged interaction cycles. Consequently, the controller design transfers to other temporal KGQA settings with minimal adaptation, delivering both quality and efficiency without dataset-specific heuristics.

\section{Conclusions}

In this paper, a new RAG framework is proposed, i.e., Knowledge Graph-Based Iterative Retrieval-Augmented Generation\textbf{(KG-IRAG)}, which enhances the integration of KG with LLMs through iterative reasoning and retrieval. Unlike traditional RAG methods, KG-IRAG flexibly and effectively employs a step-by-step retrieval mechanism that guides LLMs in determining when to stop exploration, significantly improving response accuracy. To evaluate the framework’s effectiveness, three new datasets,~\textbf{weatherQA-Irish},~\textbf{weatherQA-Sydney}, and~\textbf{trafficQA-TFNSW}, have been introduced. These datasets involve real-world scenarios, designed to test the ability of LLMs to handle time-sensitive and event-based queries requiring both temporal reasoning and logical inference. Experimental results with proposed benchmark and general temporal ones demonstrate that KG-IRAG excels in complex reasoning tasks by generating more accurate and efficient responses.

\section{Acknowledgements}

This research is partially support by Technology Innovation Institure Abu Dhabi, UAE. Also, this research includes computations using the computational cluster Wolfpack supported by School of Computer Science and Engineering at UNSW Sydney.


\bibliographystyle{ACM-Reference-Format}
\balance
\bibliography{custom}

@article{sanmartin2024kg,
  title={KG-RAG: Bridging the Gap Between Knowledge and Creativity},
  author={Sanmartin, Diego},
  journal={arXiv preprint arXiv:2405.12035},
  year={2024}
}

@article{gao2023retrieval,
  title={Retrieval-augmented generation for large language models: A survey},
  author={Gao, Yunfan and Xiong, Yun and Gao, Xinyu and Jia, Kangxiang and Pan, Jinliu and Bi, Yuxi and Dai, Yi and Sun, Jiawei and Wang, Haofen},
  journal={arXiv preprint arXiv:2312.10997},
  year={2023}
}

@article{zhao2024retrieval,
  title={Retrieval-augmented generation for ai-generated content: A survey},
  author={Zhao, Penghao and Zhang, Hailin and Yu, Qinhan and Wang, Zhengren and Geng, Yunteng and Fu, Fangcheng and Yang, Ling and Zhang, Wentao and Cui, Bin},
  journal={arXiv preprint arXiv:2402.19473},
  year={2024}
}

@inproceedings{fan2024survey,
  title={A survey on rag meeting llms: Towards retrieval-augmented large language models},
  author={Fan, Wenqi and Ding, Yujuan and Ning, Liangbo and Wang, Shijie and Li, Hengyun and Yin, Dawei and Chua, Tat-Seng and Li, Qing},
  booktitle={Proceedings of the 30th ACM SIGKDD Conference on Knowledge Discovery and Data Mining},
  pages={6491--6501},
  year={2024}
}

@article{yu2024evaluation,
  title={Evaluation of Retrieval-Augmented Generation: A Survey},
  author={Yu, Hao and Gan, Aoran and Zhang, Kai and Tong, Shiwei and Liu, Qi and Liu, Zhaofeng},
  journal={arXiv preprint arXiv:2405.07437},
  year={2024}
}

@article{edge2024local,
  title={From local to global: A graph rag approach to query-focused summarization},
  author={Edge, Darren and Trinh, Ha and Cheng, Newman and Bradley, Joshua and Chao, Alex and Mody, Apurva and Truitt, Steven and Larson, Jonathan},
  journal={arXiv preprint arXiv:2404.16130},
  year={2024}
}

@article{hu2024grag,
  title={GRAG: Graph Retrieval-Augmented Generation},
  author={Hu, Yuntong and Lei, Zhihan and Zhang, Zheng and Pan, Bo and Ling, Chen and Zhao, Liang},
  journal={arXiv preprint arXiv:2405.16506},
  year={2024}
}

@article{peng2024graph,
  title={Graph retrieval-augmented generation: A survey},
  author={Peng, Boci and Zhu, Yun and Liu, Yongchao and Bo, Xiaohe and Shi, Haizhou and Hong, Chuntao and Zhang, Yan and Tang, Siliang},
  journal={arXiv preprint arXiv:2408.08921},
  year={2024}
}

@article{cui2019kbqa,
  title={KBQA: learning question answering over QA corpora and knowledge bases},
  author={Cui, Wanyun and Xiao, Yanghua and Wang, Haixun and Song, Yangqiu and Hwang, Seung-won and Wang, Wei},
  journal={arXiv preprint arXiv:1903.02419},
  year={2019}
}

@inproceedings{wu2019survey,
  title={A survey of question answering over knowledge base},
  author={Wu, Peiyun and Zhang, Xiaowang and Feng, Zhiyong},
  booktitle={Knowledge Graph and Semantic Computing: Knowledge Computing and Language Understanding: 4th China Conference, CCKS 2019, Hangzhou, China, August 24--27, 2019, Revised Selected Papers 4},
  pages={86--97},
  year={2019},
  organization={Springer}
}

@article{zhang2024causal,
  title={Causal graph discovery with retrieval-augmented generation based large language models},
  author={Zhang, Yuzhe and Zhang, Yipeng and Gan, Yidong and Yao, Lina and Wang, Chen},
  journal={arXiv preprint arXiv:2402.15301},
  year={2024}
}

@article{bustamante2024sparql,
  title={SPARQL Generation with Entity Pre-trained GPT for KG Question Answering},
  author={Bustamante, Diego and Takeda, Hideaki},
  journal={arXiv preprint arXiv:2402.00969},
  year={2024}
}

@article{chen2019tabfact,
  title={Tabfact: A large-scale dataset for table-based fact verification},
  author={Chen, Wenhu and Wang, Hongmin and Chen, Jianshu and Zhang, Yunkai and Wang, Hong and Li, Shiyang and Zhou, Xiyou and Wang, William Yang},
  journal={arXiv preprint arXiv:1909.02164},
  year={2019}
}

@article{fu2020survey,
  title={A survey on complex question answering over knowledge base: Recent advances and challenges},
  author={Fu, Bin and Qiu, Yunqi and Tang, Chengguang and Li, Yang and Yu, Haiyang and Sun, Jian},
  journal={arXiv preprint arXiv:2007.13069},
  year={2020}
}

@article{suresh2024overview,
  title={Overview of Factify5WQA: Fact Verification through 5W Question-Answering},
  author={Suresh, Suryavardan and Rani, Anku and Patwa, Parth and Reganti, Aishwarya and Jain, Vinija and Chadha, Aman and Das, Amitava and Sheth, Amit and Ekbal, Asif},
  journal={arXiv preprint arXiv:2410.04236},
  year={2024}
}

@article{jiang2024kg,
  title={Kg-agent: An efficient autonomous agent framework for complex reasoning over knowledge graph},
  author={Jiang, Jinhao and Zhou, Kun and Zhao, Wayne Xin and Song, Yang and Zhu, Chen and Zhu, Hengshu and Wen, Ji-Rong},
  journal={arXiv preprint arXiv:2402.11163},
  year={2024}
}

@article{jiang2022unikgqa,
  title={Unikgqa: Unified retrieval and reasoning for solving multi-hop question answering over knowledge graph},
  author={Jiang, Jinhao and Zhou, Kun and Zhao, Wayne Xin and Wen, Ji-Rong},
  journal={arXiv preprint arXiv:2212.00959},
  year={2022}
}

@article{chakraborty2024multi,
  title={Multi-hop Question Answering over Knowledge Graphs using Large Language Models},
  author={Chakraborty, Abir},
  journal={arXiv preprint arXiv:2404.19234},
  year={2024}
}

@article{fang2024dara,
  title={DARA: Decomposition-Alignment-Reasoning Autonomous Language Agent for Question Answering over Knowledge Graphs},
  author={Fang, Haishuo and Zhu, Xiaodan and Gurevych, Iryna},
  journal={arXiv preprint arXiv:2406.07080},
  year={2024}
}

@article{chen2023label,
  title={Label-free node classification on graphs with large language models (llms)},
  author={Chen, Zhikai and Mao, Haitao and Wen, Hongzhi and Han, Haoyu and Jin, Wei and Zhang, Haiyang and Liu, Hui and Tang, Jiliang},
  journal={arXiv preprint arXiv:2310.04668},
  year={2023}
}

@article{shu2024knowledge,
  title={Knowledge Graph Large Language Model (KG-LLM) for Link Prediction},
  author={Shu, Dong and Chen, Tianle and Jin, Mingyu and Zhang, Yiting and Du, Mengnan and Zhang, Yongfeng},
  journal={arXiv preprint arXiv:2403.07311},
  year={2024}
}

@inproceedings{yuan2024back,
  title={Back to the future: Towards explainable temporal reasoning with large language models},
  author={Yuan, Chenhan and Xie, Qianqian and Huang, Jimin and Ananiadou, Sophia},
  booktitle={Proceedings of the ACM on Web Conference 2024},
  pages={1963--1974},
  year={2024}
}

@article{xiong2024large,
  title={Large language models can learn temporal reasoning},
  author={Xiong, Siheng and Payani, Ali and Kompella, Ramana and Fekri, Faramarz},
  journal={arXiv preprint arXiv:2401.06853},
  year={2024}
}

@article{chen2021dataset,
  title={A dataset for answering time-sensitive questions},
  author={Chen, Wenhu and Wang, Xinyi and Wang, William Yang},
  journal={arXiv preprint arXiv:2108.06314},
  year={2021}
}

@article{dhingra2022time,
  title={Time-aware language models as temporal knowledge bases},
  author={Dhingra, Bhuwan and Cole, Jeremy R and Eisenschlos, Julian Martin and Gillick, Daniel and Eisenstein, Jacob and Cohen, William W},
  journal={Transactions of the Association for Computational Linguistics},
  volume={10},
  pages={257--273},
  year={2022},
  publisher={MIT Press One Broadway, 12th Floor, Cambridge, Massachusetts 02142, USA~…}
}

@article{jang2022temporalwiki,
  title={Temporalwiki: A lifelong benchmark for training and evaluating ever-evolving language models},
  author={Jang, Joel and Ye, Seonghyeon and Lee, Changho and Yang, Sohee and Shin, Joongbo and Han, Janghoon and Kim, Gyeonghun and Seo, Minjoon},
  journal={arXiv preprint arXiv:2204.14211},
  year={2022}
}

@article{tan2023towards,
  title={Towards benchmarking and improving the temporal reasoning capability of large language models},
  author={Tan, Qingyu and Ng, Hwee Tou and Bing, Lidong},
  journal={arXiv preprint arXiv:2306.08952},
  year={2023}
}

@article{lee2023temporal,
  title={Temporal knowledge graph forecasting without knowledge using in-context learning},
  author={Lee, Dong-Ho and Ahrabian, Kian and Jin, Woojeong and Morstatter, Fred and Pujara, Jay},
  journal={arXiv preprint arXiv:2305.10613},
  year={2023}
}

@article{xia2024enhancing,
  title={Enhancing temporal knowledge graph forecasting with large language models via chain-of-history reasoning},
  author={Xia, Yuwei and Wang, Ding and Liu, Qiang and Wang, Liang and Wu, Shu and Zhang, Xiaoyu},
  journal={arXiv preprint arXiv:2402.14382},
  year={2024}
}

@article{li2023chain,
  title={Chain-of-knowledge: Grounding large language models via dynamic knowledge adapting over heterogeneous sources},
  author={Li, Xingxuan and Zhao, Ruochen and Chia, Yew Ken and Ding, Bosheng and Joty, Shafiq and Poria, Soujanya and Bing, Lidong},
  journal={arXiv preprint arXiv:2305.13269},
  year={2023}
}

@article{he2024g,
  title={G-retriever: Retrieval-augmented generation for textual graph understanding and question answering},
  author={He, Xiaoxin and Tian, Yijun and Sun, Yifei and Chawla, Nitesh V and Laurent, Thomas and LeCun, Yann and Bresson, Xavier and Hooi, Bryan},
  journal={arXiv preprint arXiv:2402.07630},
  year={2024}
}

@article{zhang2024raft,
  title={Raft: Adapting language model to domain specific rag},
  author={Zhang, Tianjun and Patil, Shishir G and Jain, Naman and Shen, Sheng and Zaharia, Matei and Stoica, Ion and Gonzalez, Joseph E},
  journal={arXiv preprint arXiv:2403.10131},
  year={2024}
}

@article{siriwardhana2023improving,
  title={Improving the domain adaptation of retrieval augmented generation (RAG) models for open domain question answering},
  author={Siriwardhana, Shamane and Weerasekera, Rivindu and Wen, Elliott and Kaluarachchi, Tharindu and Rana, Rajib and Nanayakkara, Suranga},
  journal={Transactions of the Association for Computational Linguistics},
  volume={11},
  pages={1--17},
  year={2023},
  publisher={MIT Press One Broadway, 12th Floor, Cambridge, Massachusetts 02142, USA~…}
}

@article{wu2024medical,
  title={Medical graph rag: Towards safe medical large language model via graph retrieval-augmented generation},
  author={Wu, Junde and Zhu, Jiayuan and Qi, Yunli},
  journal={arXiv preprint arXiv:2408.04187},
  year={2024}
}

@article{matsumoto2024kragen,
  title={KRAGEN: a knowledge graph-enhanced RAG framework for biomedical problem solving using large language models},
  author={Matsumoto, Nicholas and Moran, Jay and Choi, Hyunjun and Hernandez, Miguel E and Venkatesan, Mythreye and Wang, Paul and Moore, Jason H},
  journal={Bioinformatics},
  volume={40},
  number={6},
  year={2024},
  publisher={Oxford Academic}
}

@article{zhang2023siren,
  title={Siren's song in the AI ocean: a survey on hallucination in large language models},
  author={Zhang, Yue and Li, Yafu and Cui, Leyang and Cai, Deng and Liu, Lemao and Fu, Tingchen and Huang, Xinting and Zhao, Enbo and Zhang, Yu and Chen, Yulong and others},
  journal={arXiv preprint arXiv:2309.01219},
  year={2023}
}

@inproceedings{kandpal2023large,
  title={Large language models struggle to learn long-tail knowledge},
  author={Kandpal, Nikhil and Deng, Haikang and Roberts, Adam and Wallace, Eric and Raffel, Colin},
  booktitle={International Conference on Machine Learning},
  pages={15696--15707},
  year={2023},
  organization={PMLR}
}

@article{gao2023enabling,
  title={Enabling large language models to generate text with citations},
  author={Gao, Tianyu and Yen, Howard and Yu, Jiatong and Chen, Danqi},
  journal={arXiv preprint arXiv:2305.14627},
  year={2023}
}

@article{li2024sensorllm,
  title={Sensorllm: Aligning large language models with motion sensors for human activity recognition},
  author={Li, Zechen and Deldari, Shohreh and Chen, Linyao and Xue, Hao and Salim, Flora D},
  journal={arXiv preprint arXiv:2410.10624},
  year={2024}
}

@inproceedings{cuconasu2024power,
  title={The power of noise: Redefining retrieval for rag systems},
  author={Cuconasu, Florin and Trappolini, Giovanni and Siciliano, Federico and Filice, Simone and Campagnano, Cesare and Maarek, Yoelle and Tonellotto, Nicola and Silvestri, Fabrizio},
  booktitle={Proceedings of the 47th International ACM SIGIR Conference on Research and Development in Information Retrieval},
  pages={719--729},
  year={2024}
}

@article{huang2022towards,
  title={Towards reasoning in large language models: A survey},
  author={Huang, Jie and Chang, Kevin Chen-Chuan},
  journal={arXiv preprint arXiv:2212.10403},
  year={2022}
}

@inproceedings{jia2021complex,
  title={Complex temporal question answering on knowledge graphs},
  author={Jia, Zhen and Pramanik, Soumajit and Saha Roy, Rishiraj and Weikum, Gerhard},
  booktitle={Proceedings of the 30th ACM international conference on information \& knowledge management},
  pages={792--802},
  year={2021}
}

@article{sun2023think,
  title={Think-on-graph: Deep and responsible reasoning of large language model on knowledge graph},
  author={Sun, Jiashuo and Xu, Chengjin and Tang, Lumingyuan and Wang, Saizhuo and Lin, Chen and Gong, Yeyun and Ni, Lionel M and Shum, Heung-Yeung and Guo, Jian},
  journal={arXiv preprint arXiv:2307.07697},
  year={2023}
}

@article{zuo2025falcon,
  title={Falcon-h1: A family of hybrid-head language models redefining efficiency and performance},
  author={Zuo, Jingwei and Velikanov, Maksim and Chahed, Ilyas and Belkada, Younes and Rhayem, Dhia Eddine and Kunsch, Guillaume and Hacid, Hakim and Yous, Hamza and Farhat, Brahim and Khadraoui, Ibrahim and others},
  journal={arXiv preprint arXiv:2507.22448},
  year={2025}
}

\appendix

\section{Notations} \label{app:notation}

Important notations are listed in Table~\ref{table:notation}.

\begin{table*}[t]
  \centering
  \caption{Key notations used in KG-iRAG.}
  \label{table:notation}
  \small
  \begin{adjustbox}{max width=0.9\textwidth}
  \begin{tabular}{c|l}
    \toprule
    \textbf{Symbol} & \textbf{Meaning} \\
    \midrule
    $S^{(t)}=(\tau^{(t)},\ell^{(t)},R^{(t)},p^{(t)})$ & Reasoning state at iteration $t$ (time anchor, location, retrieved triples, pattern) \\
    $\Psi(D;\tau,\ell,p)$ & Time-aware retrieval operator around $(\tau,\ell)$ under pattern $p$ (Eq.~\ref{eq:retrieval}) \\
    $\mathcal{W}(\tau)$, $r$ & Temporal window around anchor $\tau$; spatial radius \\
    $\mathrm{AllenConsistent}(\cdot)$ & Interval-consistency check used in the stop rule (Eq.~\ref{eq:stop}) \\
    $\mathrm{LLM1}$ / $\mathrm{LLM2}$ & Planner (anchor update) / Verifier (sufficiency + synthesis) \\
    $k_0$ & Local neighborhood hop cap per retrieval (typically $1$–$2$) \\
    $T_{\max}$ & Maximum iterations (latency budget) \\
    $SD$ & Standard/minimal evidence for a question (Sec.~\ref{data:qa}) \\
    $\Delta t$, $L$ & Trip duration and search horizon for Q2/Q3 labels \\
    $\delta$ & Effective branching factor for graph expansion beyond the time line \\
    $d^\star$ & Temporal offset (in window-steps) from initial anchor to the nearest feasible window \\
    \bottomrule
  \end{tabular}
  \end{adjustbox}
\end{table*}

\section{Theoretical Analysis of Iterative Retrieval} \label{app:theore}

\subsection{Coverage necessity for single-pass windows}

\begin{definition}[Nearest feasible offset]
Given an anchor window $[t_0, t_0{+}\Delta t]$ and a feasibility predicate (e.g., \emph{no event during the window}), let $d^\star\!\ge\!0$ be the smallest number of window shifts (each of size $\Delta t$ or a fixed stride) required so that the shifted window $[t_0{+}d^\star\!\cdot\!\mathrm{stride},\, t_0{+}d^\star\!\cdot\!\mathrm{stride}{+}\Delta t]$ satisfies the predicate.
\end{definition}

\begin{lemma}[Single-pass coverage condition]
Any single-pass retriever that inspects a temporal radius $w$ (in the same unit as $d^\star$) can succeed only if $w \ge d^\star$. If a dataset contains instances with $d^\star > w$, the single-pass procedure necessarily misses the first feasible window unless it increases $w$.
\end{lemma}

\begin{proof}[Sketch]
A single-pass procedure fixed at radius $w$ returns evidence confined to windows whose start offsets lie in $[{-}w,{+}w]$ around the anchor. The earliest feasible window starts at offset $d^\star$. If $d^\star>w$, that window is outside the inspected range and thus cannot be certified without enlarging the radius.
\end{proof}

\subsection{Cost scaling: single-pass vs.\ iterative}

Let $\mathrm{Cost}_{\mathrm{temp}}(w)$ denote the temporal scan cost for radius $w$ (linear in $w$), and let graph-expansion cost at hop $k$ scale as $\Theta(\delta^{k})$ with effective branching $\delta\!>\!1$ when joining metadata, nearby entities, or derived edges. A single-pass strategy that guarantees success on all instances up to offset $d^\star$ must set $w\!\ge\!d^\star$, yielding
\[
\mathrm{Cost}_{\mathrm{single}} \;=\; \Theta\!\big(\mathrm{Cost}_{\mathrm{temp}}(d^\star)\big)\;+\;\Theta\!\big(\delta^{\,k}\big).
\]
In contrast, the time-anchored iterative policy uses small fixed $k_0\ll k$ and advances anchors in $T\!\approx\!d^\star$ bounded steps until the feasibility predicate holds:
\[
\mathrm{Cost}_{\mathrm{iter}} \;=\; T\cdot\Big(\Theta\!\big(\mathrm{Cost}_{\mathrm{temp}}(1)\big)\;+\;\Theta\!\big(\delta^{\,k_0}\big)\Big).
\]
Since $\mathrm{Cost}_{\mathrm{temp}}$ is linear and $\delta^{k}$ grows super-linearly in $k$ for $\delta\!>\!1$, using a small per-step hop $k_0$ avoids the exponential blow-up required by a single-pass guarantee at large $d^\star$ while still covering the necessary offset through controlled iteration.

\subsection{Precision monotonicity with window size}

\begin{proposition}[Context dilution]
Assume irrelevant facts arrive at an average rate $\lambda_{\mathrm{irr}}$ per temporal unit near the anchor, while the number of relevant facts required to certify feasibility is bounded by a constant $c$ (e.g., a small set of interval facts). The expected precision of a single-pass window of radius $w$ then satisfies
\[
\mathbb{E}[\mathrm{Precision}(w)] \;=\; \frac{c}{\,c + \Theta(\lambda_{\mathrm{irr}}\cdot w)\,},
\]
which is non-increasing in $w$. Iterative retrieval that stops immediately upon feasibility keeps $w$ small at the stopping step and therefore attains precision at least as high as any single-pass run that uses a larger $w$ to guarantee coverage.
\end{proposition}

\begin{proof}[Sketch]
Under the assumption, relevant mass is $c$ while irrelevant mass grows proportionally with window size. Hence precision decreases with $w$. The iterative stop rule halts at the smallest window that satisfies feasibility (via $\mathrm{AllenConsistent}$), bounding the final context and preventing further dilution.
\end{proof}

\paragraph{Consequence.}
The pair of properties—(i) coverage through measured anchor advances and (ii) bounded final window via a verifiable stop rule—explains the empirical gains of KG-iRAG on Q2/Q3: it achieves the necessary temporal reach without incurring the single-pass cost at large radii and avoids context dilution that elevates hallucination.

\section{Experiment on QA datasets-Direct Data Input}

In the first round of experiments, four LLMs are used to test how LLMs can directly answer the constructed three QA datasets (~\textit{weatherQA-Irish},~\textit{weatherQA-Sydney}, and~\textit{trafficQA-TFNSW}), with some settings:

1) To ensure LLMs are not interrupted by useless information, input prompts only contain questions as well as \textbf{least needed data}.

2) To test the effeciency of data input, three formats of data are used: raw data (table) format, text data (by transferring data into text description), and triplet format (KG structure). Tests using threee formats of data are done separately.

3) The final answer are compared directly with correct answers, the Accuracy(Acc) value is shown in table~\ref{data_comparison}.

\section{Experiment on QA datasets-Different RAG Methods}

In second stage of experiments, KG-IRAG is compared with Graph-RAG and KG-RAG~\citet{sanmartin2024kg}. None data are provided in the beginning, data retrieval plan are generated based on different framework. The Accuracy, like last stage of experiment, is the result of comparison between generated final answer and true answer. F1 Score and Hit Rate focus more on the whether the retrieval is both enough and non-excessive. Hallucinations are judged based on answers generated by LLMs under different framework. For Q1, F1 Score and Hit Rate is not considered, since it contain less temporal reasoning compared with Q2 and Q3. The results of WeatherQA-Irish are shown in table~\ref{model_comparison_Irish}, while WeatherQA-Sydney results are shown in table~\ref{model_comparison_Sydney}, and those of TrafficQA-TFNSW are in table~\ref{model_comparison_TFNSW}.

\begin{table}[h]
\small
\centering
\caption{Attributes used for KG construction and QA labeling.}
\begin{adjustbox}{max width=0.47\textwidth}
\begin{tabular}{ll}
\toprule
\multicolumn{2}{c}{\textbf{Irish Weather}} \\
\midrule
attribute & explanation \\
\midrule
date & date/time of record (hourly) \\
county & county name \\
station ID & station identifier \\
station & station name \\
rain & rainfall volume per hour \\
\midrule
\multicolumn{2}{c}{\textbf{Sydney Weather}} \\
\midrule
attribute & explanation \\
\midrule
date & date/time of record (30-min) \\
station ID & station identifier \\
station & station name \\
rain & rainfall volume per 30 minutes \\
\midrule
\multicolumn{2}{c}{\textbf{TFNSW Traffic}} \\
\midrule
attribute & explanation \\
\midrule
date & date/time of record (hourly) \\
direction & cardinal direction (northbound/southbound) \\
volume & traffic count per hour \\
\bottomrule
\end{tabular}
\end{adjustbox}
\label{attribute-combined}
\end{table}

\begin{table*}[h]
\centering
\caption{Comparison of Direct Data Inputs on Three Datasets Based on Accuracy(Acc) }
\begin{adjustbox}{max width=\textwidth}
\begin{tabular}{cc|ccc|ccc|ccc|ccc}
    \toprule
    \multirow{2}{*}{Data} & \multirow{1}{*}{} & \multicolumn{3}{c}{Llama-3-8b}  & \multicolumn{3}{c}{GPT-3.5-turbo}  & \multicolumn{3}{c}{GPT-4o-mini} & \multicolumn{3}{c}{GPT-4o} \\ 
    \cmidrule(lr){3-5} \cmidrule(lr){6-8} \cmidrule(lr){9-11} \cmidrule(lr){12-14}
                           &                     & Raw  & Text  & Triplet   & Raw  & Text  & Triplet  & Raw  & Text  & Triplet & Raw  & Text  & Triplet \\ 
    \midrule
    \multirow{3}{*}{\shortstack{weatherQA-\\Irish}} 
                           & Q1 &0.905 &0.9016 &\textbf{0.9316}  
                                &0.9483 &0.9417 &\textbf{0.9783}   
                                &0.9817 &0.9833 &\textbf{0.9933}   
                                &0.9950 &0.9967 &\textbf{1}\\ 
                           & Q2 &0.1159 &0.1159 &\textbf{0.1256}   
                                &\textbf{0.14} &0.0966 &0.1111   
                                &0.1932 &0.1498 &\textbf{0.2415}   
                                &0.3865 &0.3314 &\textbf{0.396} \\ 
                           & Q3 &0.3233 &\textbf{0.343} &0.3382   
                                &0.2946 &0.2995 &\textbf{0.314}   
                                &0.4203 &0.4348 &\textbf{0.4492}   
                                &0.4831 &0.5072 &\textbf{0.5169} \\
    \midrule
    \multirow{3}{*}{\shortstack{weatherQA-\\Sydney}}  
                           & Q1 &\textbf{0.8867} &0.855 &0.8683   
                                &0.9117 &0.7083 &\textbf{0.9383}   
                                &\textbf{1} &0.9667 &0.9717   
                                &0.9867 &0.9667 &\textbf{0.99} \\ 
                           & Q2 &\textbf{0.1542} &0.1285 &0.1491   
                                &0.1259 &0.162 &\textbf{0.1722}   
                                &0.2314 &0.2596 &\textbf{0.3033}   
                                &0.2545 &0.2879 &\textbf{0.347} \\ 
                           & Q3 &0.1722 &0.1517 &\textbf{0.1877}   
                                &0.1671 &0.2108 &0.2005   
                                &0.2416 &\textbf{0.2699} &\textbf{0.2699}   
                                &0.2956 &\textbf{0.3316} &0.3265 \\
    \midrule
    \multirow{3}{*}{\shortstack{trafficQA-\\TFNSW}}  
                           & Q1 &0.52 &0.515 &\textbf{0.5325}   
                                &\textbf{0.4125} &0.395 &0.41   
                                &0.5975 &0.6475 &\textbf{0.73}  
                                &0.95 &0.965 &\textbf{0.9725} \\ 
                           & Q2 &0.1633 &\textbf{0.1837} &0.1701   
                                &0.1293 &0.1293 &\textbf{0.1565} 
                                &0.2721 &0.2857 &\textbf{0.3061}   
                                &0.4354 &0.381 &\textbf{0.4762} \\ 
                           & Q3 &\textbf{0.2381} &0.1973 &0.2313   
                                &\textbf{0.2109} &0.2041 &\textbf{0.2109}   
                                &0.4014 &0.381 &\textbf{0.4285}   
                                &0.6871 &0.66 &\textbf{0.7279} \\
    \midrule
\end{tabular}
\end{adjustbox}
\label{data_comparison}
\end{table*}

\begin{table*}[h]
\centering
\caption{Comparison of Different RAG Methods on WeatherQA-Irish Dataset}
\begin{adjustbox}{max width=\textwidth}
\begin{tabular}{cc|cc|cccc|cccc}
    \toprule
    \multicolumn{12}{c}{\textbf{WeatherQA-Irish}} \\ \cline{1-12}
    \multirow{2}{*}{Model} & \multirow{2}{*}{Data Type} & \multicolumn{2}{c|}{Question 1} & \multicolumn{4}{c|}{Question 2} & \multicolumn{4}{c}{Question 3} \\ 
    \cmidrule(lr){3-4} \cmidrule(lr){5-8} \cmidrule(lr){9-12}
                           &                     & Acc   & Hall. & Acc  & F1 Score  & HR  & Hall. & Acc  & F1 Score  & HR  & Hall. \\ 
    \midrule
    \multirow{3}{*}{Llama-3-8b} & Graph-RAG                 
                                &0.916 &0.086
                                &0.063 &0.673 &0.509 &0.5
                                &0.15  &0.684 &0.517 &0.48     \\ 
                                & KG-RAG                 
                                &0.892 &0.108
                                &0.217 &0.815 &0.694 &0.34
                                &0.386 &0.809 &0.686 &\textbf{0.22}     \\ 
                                & \textbf{KG-IRAG}                 
                                &\textbf{0.927} &\textbf{0.081}
                                &\textbf{0.242} &\textbf{0.821} &\textbf{0.702} &\textbf{0.32}
                                &\textbf{0.415} &\textbf{0.831} &\textbf{0.715} &\textbf{0.22}     \\ 
                           
    \midrule
    \multirow{3}{*}{GPT-3.5-turbo} & Graph-RAG                 
                                &\textbf{0.958} &0.042
                                &0.087 &0.673 &0.509 &0.48
                                &0.14  &0.684 &0.517 &0.42     \\ 
                                & KG-RAG                 
                                &0.945 &0.055
                                &0.208 &0.807 &0.683 &\textbf{0.32}
                                &0.362 &0.803 &0.677 &\textbf{0.18}     \\ 
                                & \textbf{KG-IRAG}                 
                                &\textbf{0.958} &\textbf{0.036}
                                &\textbf{0.222} &\textbf{0.817} &\textbf{0.697} &\textbf{0.32}
                                &\textbf{0.386} &\textbf{0.824} &\textbf{0.703} &0.2     \\ 
    \midrule
    \multirow{3}{*}{GPT-4o-mini} & Graph-RAG                 
                                &0.97 &0.03
                                &0.145 &0.673 &0.509 &0.32
                                &0.251 &0.684 &0.517 &0.28     \\ 
                                & KG-RAG                 
                                &0.975 &\textbf{0.025}
                                &0.304 &0.832 &0.718 &0.24
                                &0.546 &0.813 &0.695 &0.18     \\ 
                                & \textbf{KG-IRAG}                 
                                &\textbf{0.983} &0.028
                                &\textbf{0.323} &\textbf{0.84} &\textbf{0.721} &\textbf{0.22}
                                &\textbf{0.585} &\textbf{0.859} &\textbf{0.754} &\textbf{0.16}     \\ 
    \midrule
    \multirow{4}{*}{GPT-4o}   & Graph-RAG                 
                                &0.983 &0.017
                                &0.208 &0.673 &0.509 &0.24
                                &0.357 &0.684 &0.517 &0.22     \\ 
                                & KG-RAG                 
                                &\textbf{0.998} &\textbf{0.002}
                                &0.435 &0.857 &0.751 &\textbf{0.14}
                                &0.575 &0.839 &0.728 &0.16     \\ 
                                & ToG                 
                                &0.996 & 0.004
                                &0.449 &0.842 &0.770 &0.15
                                &0.584 &0.857 &0.736 &0.15     \\ 
                                & \textbf{KG-IRAG}                 
                                &\textbf{0.998} &\textbf{0.002}
                                &\textbf{0.459} &\textbf{0.875} &\textbf{0.783} &0.18
                                &\textbf{0.604} &\textbf{0.88} &\textbf{0.789} &\textbf{0.14}     \\ 
                                
    \midrule
    
\end{tabular}
\end{adjustbox}
\label{model_comparison_Irish}
\end{table*}

\begin{table*}[h]
\centering
\caption{Comparison of Different RAG methods on WeatherQA-Sydney Dataset}
\begin{adjustbox}{max width=\textwidth} 
\begin{tabular}{cc|cc|cccc|cccc}
    \toprule
    \multicolumn{12}{c}{\textbf{WeatherQA-Sydney}} \\ \cline{1-12}
    \multirow{2}{*}{Model} & \multirow{2}{*}{Data Type} & \multicolumn{2}{c|}{Question 1} & \multicolumn{4}{c|}{Question 2} & \multicolumn{4}{c}{Question 3} \\ 
    \cmidrule(lr){3-4} \cmidrule(lr){5-8} \cmidrule(lr){9-12}
                           &                     & Acc  & Hall. & Acc  & F1 Score  & HR  & Hall. & Acc  & F1 Score  & HR  & Hall. \\ 
    \midrule
    \multirow{3}{*}{Llama-3-8b} & Graph-RAG                 
                                &0.877 &0.123
                                &0.141 &0.514 &0.349 &0.44
                                &0.167 &0.503 &0.326 &0.42     \\ 
                                & KG-RAG                 
                                &0.908 &0.092
                                &0.183 &0.762 &0.63 &0.26
                                &0.226 &0.741 &0.606 &\textbf{0.24}     \\ 
                                & \textbf{KG-IRAG}                 
                                &\textbf{0.916} &\textbf{0.084}
                                &\textbf{0.192} &\textbf{0.841} &\textbf{0.721} &\textbf{0.24}
                                &\textbf{0.239} &\textbf{0.833} &\textbf{0.711} &0.26     \\ 
                           
    \midrule
    \multirow{3}{*}{GPT-3.5-turbo} & Graph-RAG                 
                                &0.892 &0.108
                                &0.139 &0.514 &0.349 &0.46
                                &0.157 &0.503 &0.326 &0.42     \\ 
                                & KG-RAG                 
                                &0.933 &0.067
                                &0.205 &0.751 &0.617 &\textbf{0.24}
                                &0.231 &0.738 &0.602 &\textbf{0.22}     \\ 
                                & \textbf{KG-IRAG}                 
                                &\textbf{0.938} &\textbf{0.062}
                                &\textbf{0.246} &\textbf{0.837} &\textbf{0.716} &0.22
                                &\textbf{0.257} &\textbf{0.844} &\textbf{0.725} &\textbf{0.22}     \\ 
    \midrule
    \multirow{3}{*}{GPT-4o-mini} & Graph-RAG                 
                                &0.947 &0.053
                                &0.195 &0.514 &0.349 &0.3
                                &0.188 &0.503 &0.326 &0.36     \\ 
                                & KG-RAG                 
                                &\textbf{0.973} &0.027
                                &0.29  &0.775 &0.644 &0.18
                                &0.337 &0.78 &0.651 &\textbf{0.2}     \\ 
                                & \textbf{KG-IRAG}                 
                                &0.967 &\textbf{0.033}
                                &\textbf{0.352} &\textbf{0.862} &\textbf{0.761} &\textbf{0.16}
                                &\textbf{0.362} &\textbf{0.856} &\textbf{0.757} &\textbf{0.2}     \\ 
    \midrule
    \multirow{4}{*}{GPT-4o}   & Graph-RAG                
                                &0.968 &0.032
                                &0.226 &0.514 &0.349 &0.26
                                &0.25  &0.503 &0.326 &0.22     \\ 
                                & KG-RAG                 
                                &0.978 &0.022
                                &0.378 &0.785 &0.657 &0.16
                                &0.362 &0.793 &0.663 &0.14     \\
                                & ToG                 
                                &0.98 & \textbf{0.009}
                                &0.411 &0.833 &0.726 &0.15
                                &0.398 &0.824 &0.743 &\textbf{0.12}     \\
                                & \textbf{KG-IRAG}                 
                                &\textbf{0.983} &0.017
                                &\textbf{0.429} &\textbf{0.852} &\textbf{0.75} &\textbf{0.12}
                                &\textbf{0.411} &\textbf{0.869} &\textbf{0.769} &\textbf{0.12}    \\   
    \midrule
\end{tabular}
\end{adjustbox}
\label{model_comparison_Sydney}
\end{table*}

\begin{table*}[h]
\centering
\caption{Comparison of Different RAG methods on TrafficQA-TFNSW dataset}
\begin{adjustbox}{max width=\textwidth} 
\begin{tabular}{cc|cc|cccc|cccc}
    \toprule
    \multicolumn{12}{c}{\textbf{TrafficQA-TFNSW}} \\ \cline{1-12}
    \multirow{2}{*}{Model} & \multirow{2}{*}{Data Type} & \multicolumn{2}{c|}{Question 1} & \multicolumn{4}{c|}{Question 2} & \multicolumn{4}{c}{Question 3} \\ 
    \cmidrule(lr){3-4} \cmidrule(lr){5-8} \cmidrule(lr){9-12}
                           &                     & Acc  & Hall. & Acc  & F1 Score  & HR  & Hall. & Acc  & F1 Score  & HR  & Hall. \\ 
    \midrule
    \multirow{3}{*}{Llama-3-8b} & Graph-RAG                 
                                &0.465 &0.535
                                &0.095 &0.694 &0.537 &0.48
                                &0.122 &0.709 &0.548 &0.4    \\ 
                                & KG-RAG                 
                                &0.503 &0.497
                                &0.218 &0.843 &0.725 &\textbf{0.2}
                                &0.306 &0.835 &0.716 &\textbf{0.22}     \\ 
                                & \textbf{KG-IRAG}                 
                                &\textbf{0.52} &\textbf{0.48}
                                &\textbf{0.229} &\textbf{0.863} &\textbf{0.759} &0.21
                                &\textbf{0.333} &\textbf{0.851} &\textbf{0.745} &0.26     \\ 
                           
    \midrule
    \multirow{3}{*}{GPT-3.5-turbo} & Graph-RAG                 
                                &0.418 &0.582
                                &0.075 &0.694 &0.537 &0.5
                                &0.088 &0.709 &0.548 &0.44     \\ 
                                & KG-RAG                 
                                &0.43 &0.57
                                &0.224 &0.841 &0.722 &\textbf{0.22}
                                &0.252 &0.833 &0.712 &\textbf{0.22}     \\ 
                                & \textbf{KG-IRAG}                 
                                &\textbf{0.46} &\textbf{0.54}
                                &\textbf{0.239} &\textbf{0.858} &\textbf{0.761} &\textbf{0.22}
                                &\textbf{0.279} &\textbf{0.868} &\textbf{0.766} &0.26     \\ 
    \midrule
    \multirow{3}{*}{GPT-4o-mini} & Graph-RAG                 
                                &0.635 &0.365
                                &0.238 &0.694 &0.537 &0.36
                                &0.32  &0.709 &0.548 &0.32     \\ 
                                & KG-RAG                 
                                &0.722 &0.263
                                &0.374 &0.852 &0.736 &\textbf{0.14}
                                &0.503 &0.857 &0.743 &\textbf{0.16}     \\ 
                                & \textbf{KG-IRAG}                 
                                &\textbf{0.748} &\textbf{0.251}
                                &\textbf{0.381} &\textbf{0.882} &\textbf{0.785} &0.22
                                &\textbf{0.51} &\textbf{0.894} &\textbf{0.813} &0.18     \\ 
    \midrule
    \multirow{4}{*}{GPT-4o}   & Graph-RAG                 
                                &0.93 &0.07
                                &0.361 &0.694 &0.537 &0.25
                                &0.565 &0.709 &0.548 &0.26     \\ 
                                & KG-RAG                 
                                &0.963 &0.047
                                &0.537 &0.878 &0.783 &\textbf{0.14}
                                &0.776 &0.869 &0.771 &\textbf{0.12}     \\ 
                                & ToG                 
                                &\textbf{0.967} & 0.04
                                &0.531 &0.861 &0.788 &0.18
                                &0.789 &0.894 &0.831 &0.17     \\
                                & \textbf{KG-IRAG}                 
                                &0.965 &\textbf{0.035}
                                &\textbf{0.558} &\textbf{0.886} &\textbf{0.802} &0.2
                                &\textbf{0.803} &\textbf{0.914} &\textbf{0.843} &0.16     \\   
    \midrule
\end{tabular}
\end{adjustbox}
\label{model_comparison_TFNSW}
\end{table*}


\end{document}